\documentclass{colt2014}

\usepackage{times}
\usepackage{mysetup,xspace,QED}
\usepackage{mathrsfs}
\usepackage{verbatim}

\usepackage{algorithm}
\usepackage[noend]{algorithmic}

\usepackage{hyperref}

\providecommand{\Appendix}{}
\renewcommand{\Appendix}[2][?]{%
	\refstepcounter{section}%
	\vspace{\parskip}%
	{\flushright\Large\bfseries\appendixname\ \thesection: #1}%
	\vspace{\baselineskip}%
}

\renewcommand{\appendix}{%
	\renewcommand{\section}{\secdef\Appendix\Appendix}%
	\renewcommand{\thesection}{\Alph{section}}%
	\setcounter{section}{0}%
}

\newcommand{\qedhere}{\qquad\qquad\qed\qedsymbol} 

\newcommand{\tildeO}{\tilde{O}}
\renewcommand{\fakeItem}[1][$\bullet$]{\vspace{2mm}\noindent{\bf #1}\hspace{4mm}}

\title[Resourceful Contextual Bandits]{Resourceful Contextual Bandits%
\thanks{This is the full version of a paper in the {\em 26th Conf. on Learning Theory (COLT)}, 2014. The present version includes a correction for Theorem~\ref{thm:discretization}, a corollary for contextual dynamic pricing with discretization, and an updated discussion of related work. 
\vspace{2mm} \newline  The main results have been obtained while A. Badanidiyuru was a research intern at Microsoft Research New York City. A. Badanidiyuru was also partially supported by NSF grant AF-0910940 of Robert Kleinberg.}}
\coltauthor{\Name{Ashwinkumar Badanidiyuru} \Email{ashwinkumarbv@gmail.com}\\
\addr Google Research, Mountain View CA, USA.
\AND
\Name{John Langford} \Email{jcl@microsoft.com} \\
\addr Microsoft Research, New York, NY, USA.
\AND
\Name{Aleksandrs Slivkins} \Email{slivkins@microsoft.com} \\
\addr Microsoft Research, New York, NY, USA.
}

\begin{document}

\maketitle

\newcommand{\Otilde}{\tilde{O}}
\renewcommand{\eqref}[1]{Equation~(\ref{#1})}
\renewcommand{\Re}{\mathbb{R}}
\newcommand{\C}{\mC}
\newcommand{\D}{\mD}
\newcommand{\eos}{expected-outcomes tuple\xspace}
\newcommand{\eoss}{expected-outcomes tuples\xspace}
\newcommand{\support}{\mathtt{support}}

\newcommand{\Distr}[1]{\mathbf{\Delta}_{#1}} 

\newcommand{\ALG}{\texttt{ALG}\xspace}
\newcommand{\BwK}{\texttt{BwK}\xspace} 
\newcommand{\cBwK}{\ensuremath{\mathtt{RCB}}\xspace} 
\newcommand{\kMAB}{\CBalance} 
\newcommand{\CBalance}{\ensuremath{\mathtt{MixtureElimination}}\xspace} 

\newcommand{\arms}{Y}
\newcommand{\contexts}{X}
\newcommand{\DX}{\mD_{\mathtt{X}}} 
\newcommand{\empir}[1]{\widetilde{#1}}  
\newcommand{\Ave}[1]{<#1>}  

\newcommand{\PotPerf}{\mathbf{\Delta}} 
\newcommand{\F}{\PotPerf} 
\newcommand{\noiseProb}{q_0} 

\newcommand{\OPT}{\mathtt{OPT}}
\newcommand{\LPOPT}[1][\mathtt{LP}]{\OPT_{#1}}
\newcommand{\FPi}{\mF_{\Pi}} 

\newcommand{\Rew}{\mathtt{REW}} 
\newcommand{\LP}{\mathtt{LP}}

\newcommand{\stime}{\tau} 

\newcommand{\rad}{\mathtt{rad}} 
\newcommand{\chernoffC}{C_{\mathtt{rad}}} 
\newcommand{\Conv}{\ensuremath{\mathtt{Conv}}}

\newcommand{\fbt}{\frac{B}{T}} 
\newcommand{\tfbt}{\tfrac{B}{T}}

\newenvironment{lparray}%
{\begin{array}{l@{\hspace{8mm}}l@{\hspace{8mm}}l}}%
{\end{array}}
\newlength{\lplb}
\setlength{\lplb}{3mm}

\newcommand{\x}{{z}}
\newcommand{\adv}[1]{\ensuremath{\mathcal{I}_{#1}}\xspace}
\newcommand{\alladv}{\mathcal{A}}
\newcommand{\reg}{\ensuremath{\mathtt{REG}}}
\newcommand{\Null}{\ensuremath{\mathtt{null}}}

\newcommand{\empirval}[1]{\uppercase{#1}} 
\newcommand{\expval}[1]{\textbf #1} 
\newcommand{\avgval}[1]{\overline{#1}} 


\vspace{-15mm}
\noindent \emph{First version:} February 2014. \\
\emph{This version:} July 2015.
\vspace{5mm}

\begin{abstract}
We study contextual bandits with ancillary constraints on resources, which are common in real-world applications such as choosing ads or dynamic pricing of items.  We design the first algorithm for solving these problems that handles constrained resources other than time, and improves over a trivial reduction to the non-contextual case. We consider very general settings for both contextual bandits (arbitrary policy sets, \citet{policy_elim}) and bandits with resource constraints (bandits with knapsacks,  \citet{BwK-focs13}), and prove a regret guarantee with near-optimal statistical properties.
\end{abstract}

\section{Introduction}
\label{sec:intro}

\emph{Contextual bandits} is a machine learning framework in which an algorithm makes sequential decisions according to the following protocol: in each round, a context arrives, then the algorithm chooses an action from the fixed and known set of possible actions, and then the reward for this action is revealed; the reward may depend on the context, and can vary over time.  Contextual bandits is one of the prominent directions in the literature on online learning with exploration-exploitation tradeoff; many problems in this space are studied under the name \emph{multi-armed bandits}.

A canonical example of contextual bandit learning is choosing ads for a search engine.  Here, the goal is to choose the most profitable ad to display to a given user based on a search query and the available information about this user, and optimize the ad selection over time based on user feedback such as clicks. This description leaves out \emph{many} important details, one of which is that every ad is associated with a budget which constrains the maximum amount of revenue which that ad can generate. In fact, this issue is so important that in some formulations it is the primary problem \citep[e.g.,][]{spending_constraints}.

The optimal solution with budget constraints fundamentally differs from the optimal solution without constraints.  As an example, suppose that one ad has a high expected revenue but a small budget such that it can only be clicked on once.  Should this ad be used immediately? From all reasonable perspectives, the answer is ``no''.  From the user's or advertiser's perspective, we prefer that this ad be displayed for the user with the strongest interest rather than for a user who simply has more interest than in other options.  From a platform's viewpoint, it is better to have more ads in the system, since they effectively increases the price paid in a second price auction.  And from everyone's viewpoint, it is simply odd to burn out the budget of an ad as soon as it is available.  Instead, a small budget should be parceled out over time.

To address these issues, we consider a generalization of contextual bandits in which there are one or several \emph{resources} that are consumed by the algorithm. This formulation has many natural applications. \emph{Dynamic ad allocation} follows the ad example described above: here, resources correspond to advertisers' budgets. In \emph{dynamic pricing}, a store with a limited supply of items to sell can make customized offers to customers. In \emph{dynamic procurement}, a contractor with a batch of jobs and a limited budget can experiment with prices offered to the workers, e.g. workers in a crowdsourcing market. The above applications have been studied on its own, but never in models that combine contexts and limited resources.


We obtain the first known algorithm for contextual bandits with resource constraints (other than time) that improves over a trivial reduction to the non-contextual version of the problem. As such, we merge two lines of work on multi-armed bandits: contextual bandits and bandits with resource constraints. While significant progress has been achieved in each of the two lines of work (in particular, optimal solutions have been worked out for very general models), the specific approaches break down when applied to our model.

\xhdr{Our model.}
We define {\bf\em resourceful contextual bandits} (in short: \cBwK), a common generalization of two general models for contextual bandits and bandits with resource constraints: respectively, contextual bandits with arbitrary policy sets \citep[e.g.,][]{Langford-nips07,policy_elim} and bandits with knapsacks \citep{BwK-focs13}.

There are several resources that are consumed by the algorithm, with a separate budget constraint on each. (Time is one of these resources, with deterministic consumption of 1 for every action.) In each round, the algorithm receives a reward and consumes some amount of each resource, in a manner that depends on the context and the chosen action, and may be randomized. We consider a stationary environment: in each round, the context and the mapping from actions to rewards and resource consumption is sampled independently from a fixed joint distribution, called the \emph{outcome distribution}. Rewards and consumption of various resources can be correlated in an arbitrary way. The algorithm stops as soon as any constraint is violated. Initially the algorithm is given no information about the outcome distribution (except the distribution of context arrivals). In particular, expected rewards and resource consumptions are not known.

An algorithm is given a finite set $\Pi$ of \emph{policies}:  mappings from contexts to actions. We compete against algorithms that must commit to some policy in $\Pi$ before each round. Our benchmark is a hypothetical algorithm that knows the outcome distribution and makes optimal decisions given this knowledge and the restriction to policies in $\Pi$. The benchmark's expected total reward is denoted $\OPT(\Pi)$. \emph{Regret} of an algorithm is defined as $\OPT(\Pi)$ minus the algorithm's expected total reward.

For normalization, per-round rewards and resource consumptions lie in $[0,1]$. We assume that the distribution of context arrivals is known to the algorithm.

\xhdr{Discussion of the model.}
Allowing stochastic resource consumptions and arbitrary correlations between per-round rewards and per-round resource consumptions is essential: this is why our model subsumes diverse applications such as the ones discussed above,%
\footnote{For example, in dynamic pricing the algorithm receives a reward and loses an item only if the item is sold.}
and many extensions thereof. Further discussion of the application domains can be found in Appendix~\ref{app:apps}.

Intuitively, the policy set $\Pi$ consists of all policies that can possibly be learned by a given learning method, such as linear estimation or decision trees. Restricting to $\Pi$ allows meaningful performance guarantees even if competing against \emph{all} possible policies is intractable. The latter is common in real-life applications, as the set of possible contexts can be very large.

Our benchmark can change policies from one round to another without restriction.  As we prove, this is essentially equivalent in power to the best fixed \emph{distribution} over policies.  However, the best fixed policy may perform substantially worse.\footnote{The expected total reward of the best fixed policy can be half as large as that of the best distribution. This holds for several different domains including dynamic pricing / procurement, even without contexts \citep{BwK-focs13}. Note that without resource constraints, the two benchmarks are equivalent.}

\OMIT{ %
Our benchmark can change policies from one round to another without restriction.  Because the environment is stationary and stochastic, this is essentially equivalent in power to the best fixed \emph{distribution} over policies.  Note that the best fixed policy may perform substantially worse.\footnote{Even without contexts, expected total reward can be half as small twice as  \citet{BwK-focs13} consider the non-contextual version, and prove that a distribution over actions can be at least twice as good, in terms of the expected total reward, as the best fixed action. Such examples exist for several domains, including dynamic pricing and dynamic procurement.}
(Without resource constraints, the three benchmarks are equivalent.)
} 

Our stopping condition corresponds to hard constraints: an advertiser cannot exceed his budget, a store cannot sell more items than it has in stock, etc. An alternative stopping condition is to restrict the algorithm to actions that cannot possibly violate any constraint if chosen in the current round, and stop if there is no such action. This alternative is essentially equivalent to the original version.%
\footnote{Each budget constraint changes by at most one, which does not affect our regret bounds in any significant way.}
Moreover, we can w.l.o.g. allow our benchmark to use this alternative.

\xhdr{Our contributions: main algorithm.}
We design an algorithm, called \CBalance, and prove the following guarantee on its regret.

\begin{theorem}\label{thm:intro-algo}
For all \cBwK problems with $K$ actions, $d$ resources, time horizon $T$, and for all policy sets $\Pi$. Algorithm \CBalance achieves expected total reward
\begin{align}\label{eq:thm:intro-algo}
\Rew \geq \OPT(\Pi) -  O\left( 1+\tfrac{1}{B}\,\OPT(\Pi)\right)\sqrt{dKT \; \log \left( dKT\,|\Pi| \right)},
\end{align}
where
$B = \min_i B_i$ is the smallest of the resource constraints $B_1 \LDOTS B_d$.
\end{theorem}

\OMIT{ 
In Theorem~\ref{thm:intro-algo}, $\maxLP$ is, essentially, the best upper bound on $\OPT(\Pi)$ that is available to an algorithm a priori (see Section~\ref{sec:analysis} for precise definition).
We can trivially take $\maxLP \leq T$. However, in some settings it can be smaller: for instance, for dynamic pricing with limited supply one can take $\maxLP=B$.
} 

This regret guarantee is optimal in several regimes. First, we achieve an optimal square-root ``scaling" of regret: if all constraints are scaled by the same parameter $\alpha>0$, then regret scales as $\sqrt{\alpha}$. Second, if $B=T$ (i.e., there are no constraints), we recover the optimal $\tilde{O}(\sqrt{KT})$ regret. Third, we achieve $\tilde{O}(\sqrt{KT})$ regret for the important regime when $\OPT(\Pi)$ and $B$ are at least a constant fraction of $T$. In fact, \cite{BwK-focs13} provide a complimentary $\Omega(\sqrt{KT})$ lower bound for this regime, which holds in a very strong sense: for any given tuple $(K,B,\OPT(\Pi),T)$.

The $\sqrt{\log |\Pi|}$ term in Theorem~\ref{thm:intro-algo} is unavoidable \citep{policy_elim}.
The dependence on the \emph{minimum} of the constraints (rather than, say, the maximum or some weighted combination thereof) is also unavoidable \citep{BwK-focs13}.
For strongest results, one can rescale per-round rewards and per-round consumption of each resource so that they can be as high as 1.%
\footnote{E.g., if per-round consumption of some resource $i$ is deterministically at most $\tfrac{1}{10}$, then multiplying it by $10$ would effectively increase the corresponding budget $B_i$ by a factor of $10$, and hence can only improve the regret bound.}

Note that the regret bound in Theorem~\ref{thm:intro-algo} does not depend on the number of contexts, only on the number of policies in $\Pi$. In particular, it tolerates infinitely many contexts. On the other hand, if the set $X$ of contexts is not too large, we can also obtain a regret bound  with respect to the best policy among \emph{all} possible policies. Formally, take $\Pi = \{\text{all policies}\}$ and observe that
    $|\Pi| \leq K^{|X|}$.

Further, Theorem~\ref{thm:intro-algo} extends to policy sets $\Pi$ that consist of \emph{randomized policies}: mappings from contexts to distributions over actions. This may significantly reduce $|\Pi|$, as a given randomized policy might not be representable as a distribution over a small number of deterministic policies.%
\footnote{We can reduce \cBwK with randomized policies to \cBwK with deterministic policies simply by replacing each context $x$ with a vector $(a_{(x,\,\pi)}: \pi\in\Pi)$ such that $a_{(x,\,\pi)} = \pi(x)$, and encoding the randomization in policies through the randomization in the context arrivals. While this blows up the context space, it does not affect our regret bound.}
We assume deterministic policies in the rest of the paper.

\vspace{2mm}{\em Computational issues.} This paper is focused on proving the existence of solutions to this
problem, and the mathematical properties of such a solution.
The algorithm is specified as a mathematically
well-defined mapping from histories to actions; we do not provide a computationally efficient implementation.
Such ``information-theoretical'' results are common for the first solutions to new, broad problem formulations
~\citep[e.g.][]{LipschitzMAB-stoc08,DichotomyMAB-soda10,policy_elim}.
In particular, in the prior work for \cBwK without resource constraints there exists an algorithm with $\tilde{O}(\sqrt{KT})$ regret \citep{bandits-exp3,policy_elim}, but for all known \emph{computationally efficient} algorithms regret scales with $T$ as $T^{2/3}$ \citep{Langford-nips07}.

\xhdr{Our contributions: partial lower bound.}
We derive a partial lower bound: we prove that \cBwK is essentially hopeless for the regime
    $\OPT(\Pi)\leq B \leq \sqrt{KT}/2$.
The condition $\OPT(\Pi)\leq B$ is satisfied, for example, in dynamic pricing with limited supply.

\begin{theorem}\label{thm:intro-LB}
Any algorithm for \cBwK incurs regret $\Omega(\OPT(\Pi))$ in the worst case over all problem instances such that
    $\OPT(\Pi)\leq B \leq \sqrt{KT}/2$
(using the notation from Theorem~\ref{thm:intro-algo}).
\end{theorem}

The above lower bound is specific to the general (``contextual") case of \cBwK. In fact, it points to a stark difference between \cBwK and the non-contextual version: in the latter, $o(\OPT)$ regret is achievable as long as (for example) $B \geq \log T$ \citep{BwK-focs13}.

While Theorem~\ref{thm:intro-LB} is concerned with the regime of small $B$, note that in the ``opposite" regime of very large $B$, namely $B \gg \sqrt{KT}$, the regret achieved in
Theorem~\ref{thm:intro-algo} is quite low: it can be expressed as
    $\tilde{O}(\sqrt{KT}+\eps\cdot\OPT(\Pi))$,
where
    $B = \tfrac{1}{\eps} \sqrt{KT} $.

\xhdr{Our contributions: discretization.} In some applications of \cBwK, such as dynamic pricing and dynamic procurement, the action space is a continuous interval of prices. Theorem~\ref{thm:intro-algo} usefully applies whenever the policy set $\Pi$ is chosen so that the number of distinct actions used by policies in $\Pi$ is finite and small compared to $T$. (Because one can w.l.o.g. remove all other actions.) However, one also needs to handle problem instances in which the policies in $\Pi$ use prohibitively large or infinite number of actions.

We consider a paradigmatic example of \cBwK with an infinite action space: contextual dynamic pricing with a single product and prices in the $[0,1]$ interval. We derive a corollary of Theorem~\ref{thm:intro-algo} that applies to an arbitrary finite policy set $\Pi$. To the best of our knowledge, this is the first result on contextual dynamic pricing with infinite price set.

We use \emph{discretization}: we reduce the original problem to one in which actions (i.e., prices) are multiples of some carefully chosen $\eps>0$. Our approach proceeds as follows. For each $\eps>0$ and each policy $\pi$ let $\pi_\eps$ be a policy that takes the price computed by $\pi$ and rounds it down to the nearest multiple of \eps. We define the ``discretized" policy set
    $\Pi_\eps = \{ \pi_\eps: \pi\in \Pi \}$.
We use Theorem~\ref{thm:intro-algo} to obtain a regret bound relative to $\Pi_\eps$. Here the $\eps$ controls the tradeoff between the number of actions in that regret bound and the ``discretization error" of $\Pi_\eps$. Then we optimize the choice of $\eps$ to obtain the regret bound relative to $\Pi$. The technical difficulty here is to bound the discretization error in terms of $\eps$; for this purpose we assume Lipschitz demands.%
\footnote{Lipschitz demands is a common assumption in some of the prior work on (non-contextual) dynamic pricing, even with a single product \citep{BZ09,Wang-OR14}. However, the optimal algorithm for the single-product case \citep{DynPricing-ec12} does not need this assumption.}

\begin{theorem}\label{thm:discretization}
Consider contextual dynamic pricing with a single product and prices in $[0,1]$. Use standard notation: supply $B$, policy set $\Pi$ and time horizon $T$. Assume Lipschitz demands with Lipschitz constant $L$. Then algorithm \CBalance with discretized policy set $\Pi_\eps$ (defined as above) and $\eps$ suitably chosen as a function of $(B,T,L,|\Pi|)$ achieves expected total reward
\begin{align}\label{eq:discretization-final}
\Rew  \geq \OPT(\Pi) -
    O( T^{3/5}\,B^{1/5})
        \cdot \left(L \, \log \left(  T|\Pi| \right) \right)^{1/5}
\end{align}
\end{theorem}

This regret bound is most interesting for the important regime $B\geq \Omega(T)$ (studied, for example, in \citet{BZ09,BesbesZeevi-OR11,Wang-OR14}). Then regret is
    $O(T^{4/5})\,\left(L\,\log \left(  T|\Pi| \right) \right)^{1/5}$.

It is unclear whether this regret bound is optimal. When specialized to the non-contextual case, it is not optimal. The optimal regret is then $O(B^{2/3})$, even for an arbitrary budget $B$ and even  without the Lipscitz assumption \citep{DynPricing-ec12}. Extending the discretization approach beyond dynamic pricing with a single product is problematic even without contexts, see Section~\ref{sec:conclusions} for further discussion.

\xhdr{Discussion: main challenges in \cBwK.}
The central issue in bandit problems is the tradeoff between \emph{exploration}: acquiring new information, and \emph{exploitation}: making seemingly optimal decisions based on this information. In this paper, we resolve the explore-exploit tradeoff in the presence of contexts and resource constraints. Each of the three components (explore-exploit tradeoff, contexts, and resource constraints) presents its own challenges, and we need to deal with all these challenges simultaneously. Below we describe these individual challenges one by one.

A well-known naive solution for explore-exploit tradeoff, which we call \emph{pre-determined exploration}, decides in advance to allocate some rounds to exploration, and the remaining rounds to exploitation. The decisions in the exploration rounds do not depend on the observations, whereas the observations from the exploitation rounds do not impact future decisions. While this approach is simple and broadly applicable, it is typically inferior to more advanced solutions based on \emph{adaptive exploration} -- adapting the exploration schedule to the observations, so that many or all rounds serve both exploration and exploitation.%
\footnote{For example, the difference in regret between pre-determined and adaptive exploration is $\tildeO(\sqrt{KT})$ vs. $O(K\log T)$ for stochastic $K$-armed bandits, and $\tildeO(T^{3/4})$ vs. $\tildeO(B^{2/3})$ for dynamic pricing with limited supply.}
Thus, the general challenge in most explore-exploit settings is to design an appropriate adaptive exploration algorithm.

Resource constraints are difficult to handle for the following three reasons. First, an algorithm's ability to exploit is constrained by resource consumption for the purpose of exploration; the latter is stochastic and therefore difficult to predict in advance. Second, the expected per-round reward is no longer the right objective to optimize, as the action with the highest expected per-round reward could consume too much resources. Instead, one needs to take into account the expected reward over the entire time horizon. Third, with more than one constrained resource (incl. time) the best fixed policy is no longer the right benchmark; instead, the algorithm should search over \emph{distributions} over policies, which is a much larger search space.

In contextual bandit problems, an algorithm effectively chooses a policy $\pi\in \Pi$ in each round. Naively, this can be reduced to a non-contextual bandit problem in which ``actions" correspond to policies. In particular, the main results in \cite{BwK-focs13} directly apply to this reduced problem. However, the action space in the reduced problem has size $|\Pi|$; accordingly, regret scales as $\sqrt{|\Pi|}$ in the worst case. The challenge in contextual bandits is to reduce this dependence. In particular, note that we replace $\sqrt{|\Pi|}$ with $\log |\Pi|$, an exponential improvement.

\xhdr{Organization of the paper.}
We start with a survey of related work and preliminaries (Sections~\ref{sec:related}-\ref{sec:prelims}). We define the main algorithm, prove its correctness, and describe the key steps of regret analysis in Sections~\ref{sec:algorithm}-\ref{sec:analysis}. The remaining details of the regret analysis are in Section~\ref{app:regret-analysis}. We prove the lower bound in Section~\ref{sec:LB}.
We conclude with an extensive discussion of the state-of-art for \cBwK and the directions for further work (Sections~\ref{sec:conclusions}). Appendix~\ref{app:apps} contains a discussion of the main application domains for \cBwK.

\section{Related work}
\label{sec:related}

Multi-armed bandits have been studied since \cite{Thompson-1933} in Operations Research, Economics, and several branches of Computer Science, see \citep{Gittins-book11,Bubeck-survey12} for background. This paper unifies two active lines of work on bandits: contextual bandits and bandits with resource constraints.

Contextual Bandits~\citep{Auer-focs00,Langford-nips07} add
contextual side information which can be used in prediction. This is
a necessary complexity for virtually all applications of bandits since
it is far more common to have relevant contextual side information
than no such information. Several versions have been studied in the literature, see \citep{Bubeck-survey12,policy_elim,contextualMAB-colt11} for a discussion.
For contextual bandits with policy sets, there exist two broad families of solutions, based on multiplicative weight
algorithms~\citep{bandits-exp3, exp4_tighter, exp4p} or confidence
intervals~\citep{policy_elim, regressor_elim}.  We rework the confidence interval approach, incorporating and extending the ideas from the work on resource-constrained bandits~\citep{BwK-focs13}.

Prior work on resource-constrained bandits includes dynamic pricing with limited supply \citep{DynPricing-ec12,BZ09,BesbesZeevi-or12}, dynamic procurement on a budget \citep{DynProcurement-ec12,Krause-www13,Crowdsourcing-PositionPaper13}, dynamic ad allocation with advertisers' budgets \citep{AdsWithBudgets-arxiv13}, and bandits with a single deterministic resource \citep{GuhaM-icalp09,GuptaKMR-focs11,TranThanh-aaai10,TranThanh-aaai12}. \citet{BwK-focs13} define and optimally solve a common generalization of all these settings: the non-contextual version of \cBwK. An extensive discussion of these and other applications, including applications to repeated auctions and network routing, can be found in \citep{BwK-focs13}.

To the best of our knowledge, the only prior work that explicitly considered contextual bandits with resource constraints is \citep{Gyorgy-ijcai07}. This paper considers a somewhat incomparable setting with arbitrary policy sets and a single constrained resource: time, whose consumption is stochastic and depends on the context and the chosen action. \citet{Gyorgy-ijcai07} design an algorithm whose regret scales $O(f(t)\,\log t)$ for any time $t$, where $f$ is any positive diverging function and the constant in $O()$ depends on the problem instance and on $f$.

Our setting can be seen as a contextual bandit version of \emph{stochastic packing} \citep[e.g.][]{DevanurH-ec09,DevanurJSW-ec11}.
The difference is in the feedback structure: in stochastic packing, full information about each round is revealed before that round.

While we approximate our benchmark $\OPT(\Pi)$ with a linear program optimum, our algorithm and analysis are conceptually very different from the vast literature on approximately solving linear programs, and in particular from LP-based work on bandit problems such as \citet{GuhaMunagala-jacm10}.

\xhdr{Concurrent and independent work.}
\citet{AgrawalDevanur-ec14} study a model for contextual bandits with resource constraints that is incomparable with ours. The model for contexts is more restrictive: contexts do not change over time,%
\footnote{\citet{AgrawalDevanur-ec14} also claimed an extension to contexts that change over time, which has subsequently been retracted (see Footnote 1 in \citet{AgrawalDevanur-15}). This extension constitutes the main result in \citet{AgrawalDevanur-15} (which is subsequent work relative to the present paper).}
and expected outcome of each round is linear in the context. Whereas the model for rewards and resource constraints is more general: the total reward can be an arbitrary concave function of the time-averaged outcome vector $\bar{v}$, and the resource constraint states that $\bar{v}$ must belong to a given convex set (which can be arbitrary).

\section{Problem formulation and preliminaries}
\label{sec:prelims}

We consider an online setting where in each round an algorithm observes a context $x$ from a possibly infinite known set of possible contexts $X$ and chooses an action $a$ from a finite known set $A$. The world then specifies a reward $r \in [0,1]$ and the resource consumption. There are $d$ resources that can be consumed, and the resource consumption is specified by numbers $c_{i} \in [0,1]$ for each resource $i$. Thus, the world specifies the vector $(r; c_1 \LDOTS c_d)$, which we call the \emph{outcome vector}; this vector can depend on the the chosen action $a$ and the round.  There is a known hard constraint $B_i\in \Re_+$ on the consumption of each resource $i$; we call it a \emph{budget} for resource $i$.  The algorithm stops at the earliest time $\tau$ when any budget constraint is violated; its total reward is the sum of the rewards in all rounds strictly preceding $\tau$.  The goal of the algorithm is to maximize the expected total reward.

We are only interested in regret at a specific time $T$ (\emph{time horizon}) which is known to the algorithm. Formally, we model time as a specific resource with budget $T$ and a deterministic consumption of $1$ for every action. So $d\geq 2$ is the number of all resources, \emph{including time}. W.l.o.g., $B_i \leq T$ for every resource $i$.

We assume that an algorithm can choose to skip a round without doing anything. Formally, we posit a \emph{null action}: an action with 0 reward and 0 consumption of all resources except the time. This is for technical convenience, so as to enable Lemma~\ref{lm:perfect}.

\xhdr{Stochastic assumptions.} We assume that there exists an unknown
distribution $D(x,\vec{r},\vec{c_{i}})$, called the \emph{outcome distribution}, from which each round's
observations are created independently and identically, where the
vectors are indexed by individual actions.  In particular, context $x$ is
drawn from the marginal distribution $\DX(\cdot)$, and the observed reward
and resource consumptions for each action $a$ are drawn from the conditional distribution
$D(\vec{r}_a,\vec{c_{i}}_a|x)$.  We assume that the
marginal distribution over contexts $D(x)$ is known.

\OMIT{Prior work~\cite{policy_elim} has shown that this dependence is removable
with a more complex algorithm.}

\xhdr{Policy sets and the benchmark.}
An algorithm is given a finite set $\Pi$ of \emph{policies} -- mappings from contexts to actions. Our benchmark is a hypothetical algorithm that knows the outcome distribution $D$, and makes optimal decisions given this knowledge. The benchmark is restricted to policies in $\Pi$: before each round, it must commit to some policy $\pi\in \Pi$, and then
choose action $\pi(x)$ upon arrival of any given context $x$. The expected total reward of the benchmark is denoted $\OPT(\Pi)$. \emph{Regret} of an algorithm is $\OPT(\Pi)$ minus the algorithm's expected total reward.

\xhdr{Uniform budgets.} We say that the budgets are \emph{uniform} if $B_i=B$ for each resource $i$. Any problem instance can be reduced to one with uniform budgets by dividing all consumption values for every resource $i$ by $B_i/B$, where $B=\min_i B_i$. (That is tantamount to changing the units in which we measure consumption of resource $i$.) We assume uniform budgets $B$ from here on.


\xhdr{Notation.}
Let $r(\pi) = E_{(x,\vec{r})\sim D}[\vec{r}_{\pi(x)}]$ and $c_i(\pi) = E_{(x,\vec{c_i})\sim D}[\vec{c_i}_{\pi(x)}]$ be the expected per-round reward and the expected per-round consumption of resource $i$ for policy $\pi$.  Similary, define $r(P) = E_{\pi\sim P} [r(\pi)]$ and $c_i(P) = E_{\pi\sim P} [c_i(\pi)]$ as the natural extension to a distribution $P$ over policies.

The tuple
    $\mu = \left(\, (r(\pi); c_1(\pi) \LDOTS c_d(\pi)):\; \pi\in \Pi \,\right)$
is called the \emph{\eos}.

For a distribution $P$ over policies, let $P(\pi)$ is the probability that $P$ places over policy $\pi$. By a slight abuse of notation, let
    $P(a|x)=\sum_{\pi(x)=a} P(\pi)$
be the probability that $P$ places on action $a$ given context $x$. Thus, each context $x$ induces a distribution $P(\cdot|x)$ over actions.

\OMIT{
\begin{align}\label{eq:conf-rad-defn}
 \rad_t(\nu) = \sqrt{\chernoffC\, \nu /t}.
\end{align}
\begin{corollary}
In the setting of Lemma~\ref{lm:martingale-convergence}, assume $n\leq T$. Then
\begin{align*}
\textstyle
\Pr \left[ \tfrac{1}{n}\,\sum_{t=1}^n X_t\leq
    \rad(\frac{V_n}{bn},\, \frac{n}{b}) \right]
    \geq 1-(dT\, |\Pi|)^{-2}.
\end{align*}
\end{corollary}
}

\subsection{Linear approximation and the benchmark}
\label{sec:LP-approx}

We set up a linear relaxation that will be crucial throughout the paper. As a by-product, we (effectively) reduce our benchmark $\OPT(\Pi)$ to the best fixed distribution over policies.

A given distribution $P$ over policies defines an algorithm $\ALG_P$: in each round a policy $\pi$ is sampled independently from $P$, and the action $a = \pi(x)$ is chosen. The \emph{value} of $P$ is the total reward of this algorithm, in expectation over the outcome distribution.

As the value of $P$ is difficult to characterize exactly, we approximate it (generalizing the approach from \citep{DynPricing-ec12,BwK-focs13} for the non-contextual version). We use a linear approximation where all rewards and consumptions are deterministic and the time is continuous. Let
    $r(P,\mu)$
and
    $c_i(P,\mu)$
be the expected per-round reward and the expected per-round consumption of resource $i$ for policy $\pi\sim P$, given \eos $\mu$. Then the linear approximation corresponds to the solution of a simple linear program:
\begin{align}\label{eq:LP-t}
\begin{array}{lrcll}
	\textrm{Maximise}    & t\, r(P,\mu) &     & & \text{in $t \in \R$} \\
	\textrm{subject to}  & t\, c_i(P,\mu)    & \leq & B & \text{for each $i$} \\
			             & t     & \geq & 0.
\end{array}
\end{align}
The solution to this LP, which we call the \emph{LP-value} of $P$, is
\begin{align}\label{eq:LP-value}
 \LP(P,\mu) = r(P, \mu) \;\textstyle\,\min_i B/c_i(P,\mu).
\end{align}
Denote
    $\LPOPT = \sup_{P} \LP(P,\mu)$,
where the supremum is over all distributions $P$ over $\Pi$.

\begin{lemma}\label{lm:LPOPT-vs-OPT}
$\LPOPT \geq \OPT(\Pi) $.
\end{lemma}

Therefore, it suffices to compete against the best fixed distribution over $\Pi$, as approximated by $\LPOPT$, even though our benchmark $\OPT(\Pi)$ allows unrestricted changes over time. Note that proving regret bounds relative to $\LPOPT$ rather than to $\OPT(\Pi)$ only makes our results stronger.

A distribution $P$ over $\Pi$ that attains the supremum value $\LPOPT$ is called \emph{LP-optimal}. Such $P$ is called \emph{LP-perfect} if furthermore $|\support(P)|\leq d$ and $c_i(P,\mu) \leq B/T$ for each resource $i$. We find it useful to consider LP-perfect distributions throughout the paper.

\begin{lemma}\label{lm:perfect}
An LP-perfect distribution exists for any instance of \cBwK.
\end{lemma}

Lemma~\ref{lm:LPOPT-vs-OPT} and Lemma~\ref{lm:perfect} are proved for the non-contextual version of \cBwK in \citet{BwK-focs13}. The general case can be reduced to the non-contextual version via a standard reduction where actions in the new problem  correspond to policies in $\Pi$ in the original problem. For Lemma~\ref{lm:perfect}, \citet{BwK-focs13}  obtain an LP-perfect distribution by mixing an LP-optimal distribution with the ``null action"; this is why we allow the null action in the setting.

\section{The algorithm: \CBalance}
\label{sec:algorithm}

The algorithm's goal is to converge on a LP-perfect distribution over policies. The general design principle is to explore as much as possible while avoiding obviously suboptimal decisions.

\xhdr{Overview of the algorithm.} In each round $t$, the following happens.

\fakeItem[1.] \emph{Compute estimates.}
We compute high-confidence estimates for the per-round reward $r(\pi)$ and per-round consumption $c_i(\pi)$, for each policy $\pi\in \Pi$ and each resource $i$. The collection $\mI$ of all \eos that are consistent with these high-confidence estimates is called the \emph{confidence region}.

\fakeItem[2.] \emph{Avoid obviously suboptimal decisions.}
We prune away all distributions $P$ over policies in $\Pi$ that are not LP-perfect with high confidence. More precisely, we prune all $P$ that are not LP-perfect for any \eos in the confidence region $\mI$; the remaining distributions are called \emph{potentially LP-perfect}. Let $\mF$ be the convex hull of the set of all potentially LP-perfect distributions.

\fakeItem[3.] \emph{Explore as much as possible.} We choose a distribution $P\in \mF$ which is \emph{balanced}, in the sense that no action is starved; see \eqref{eq:alg-balanced} for the precise definition.  Note that balanced distributions are typically \emph{not} LP-perfect.

\fakeItem[4.] \emph{Select an action.} We choose policy $\pi \in \Pi$ independently from $P$. Given context $x$, the action $a$ is chosen as $a = \pi(x)$. The algorithm adds some random noise: with probability $\noiseProb$, the action $a$ is instead chosen uniformly at random, for some parameter $\noiseProb$.

The algorithm halts as soon as the time horizon is met, or one of the resources is exhausted.

The pseudocode can be found in Algorithm~\ref{alg:CBalance}.

\newcommand{\TAB}{\hspace{4mm}}
\floatname{algorithm}{Algorithm}
\begin{algorithm}[t]
\caption{\CBalance}
\label{alg:CBalance}
\begin{algorithmic}[1]
\STATE {\bf Parameters:} \#actions $K$, time horizon $T$, budget $B$, benchmark set $\Pi$,
    context distribution $\DX$.%
\STATE {\bf Data structure:} ``confidence region"
    $\mI \leftarrow \{ \text{all feasible \eoss} \}$.
\vspace{2mm}
\STATE {\bf For} each round $t=1\ldots T$ {\bf do}
\STATE \TAB $\PotPerf_t =  \{ \text{distributions $P$ over $\Pi$: $P$ is LP-perfect for some $\mu\in\mI$} \}$.
\STATE \TAB Let $\mF_t$ be the convex hull of $\PotPerf_t$.

\STATE \TAB Let
    $\alpha_{\pi,t} = \max_{P\in \mF_t} P(\pi), \quad\forall\pi\in\Pi$.
\STATE \TAB Choose a ``balanced" distribution $P_t\in \mF_t$:
    any $P\in \mF_t$ such that $\forall \pi\in \Pi$
    \begin{align}\label{eq:alg-balanced}
        \E_{x\sim \DX}\left[\frac{1}{(1-\noiseProb)\,P(\pi(x) |x)+\tfrac{\noiseProb}{K}}\right]\leq \frac{2K}{\alpha_{\pi,t}},
        \;\text{where}\;
        \noiseProb=\min\left(\tfrac12,\sqrt{\tfrac{K}{T}\log (K\,T|\Pi|)}\right).
    \end{align}
\vspace{-2mm}
\STATE \TAB {\bf Observe} context $x_t$; {\bf choose} action $a_t$ to "play":
\STATE \TAB\TAB with probability $\noiseProb$, draw $a_t$ u.a.r. in $A$;
                else, draw $\pi \sim P_t$ and let $a_t = \pi(x_t)$.
\STATE \TAB {\bf Observe} outcome vector ($r, c_{1} \LDOTS c_{d})$.
\STATE \TAB {\bf Halt} if one of the resources is exhausted.
\STATE \TAB Eliminate \eoss from $\mI$ that violate equations~(\ref{eq:algo-estimate-r}-\ref{eq:algo-estimate-c})
\end{algorithmic}
\end{algorithm}

\xhdr{Some details.}
 After each round $t$, we estimate the per-round consumption $c_i(\pi)$ and the per-round reward $r(\pi)$, for each policy $\pi\in \Pi$ and each resource $i$, using the following unbiased  estimators:
\begin{align*}
\empir{c}_{i}(\pi) =\frac{c_{i}\;\indicator{a = \pi(x)}}{P[a = \pi(x) \,|\, x]}
    \;\;\text{and}\;\;
\empir{r}(\pi) =\frac{r\;\indicator{a = \pi(x)}}{P[a = \pi(x) \,|\, x]}.
\end{align*}
The corresponding time-averages up to round $t$ are denoted
\begin{align*}
\hat{c}_{t,i}(\pi)  = \tfrac{1}{t-1}\, \sum_{s=1}^{t-1} \empir{c}_{s,i}(\pi)
    \;\;\text{and}\;\;
\hat{r}_t(\pi)  = \tfrac{1}{t-1}\, \sum_{s=1}^{t-1} \empir{r}_s(\pi).
\end{align*}

We show that with high probability these time-averages are close to their respective expectations. To express the confidence term in a more lucid way, we use the following shorthand, called \emph{confidence radius}:
    $\rad_t(\nu) = \sqrt{\chernoffC\, \nu /t}$,
where $\chernoffC=\Theta(\log (d\,T\,|\Pi|))$ is a parameter which we will fix later. We show that w.h.p. the following holds:
\begin{align}
|r(\pi) - \hat{r}_t(\pi)|
    &\leq \rad_t\left( K/\alpha_{\pi,t}\right),
    \label{eq:algo-estimate-r}\\
|c_i(\pi) - \hat{c}_{t,i}(\pi)|
    &\leq \rad_t\left(\, K/\alpha_{\pi,t}\right)
    \quad \text{for all $i$}.
    \label{eq:algo-estimate-c}
\end{align}
(Here
    $\alpha_{\pi,t} = \max_{P\in \mF_t} P(\pi)$,
as in Algorithm~\ref{alg:CBalance}.)

\OMIT{ 
\begin{lemma}[LP-perfect]\label{cl:LP--perfect}
Some distribution over $\Pi$ is LP--perfect.
\end{lemma}

\begin{lemma}[value]\label{cl:LP-value}
For any distribution $\mD$ over policies, its value is at most $V(\mD,\mu)$.
\end{lemma}

\begin{lemma}\label{lm:balanced}
In each round $t$, there exists a distribution $P \in \mF$ which satisfies \eqref{eq:alg-balanced}.
\end{lemma}

\begin{lemma}[confidence]\label{lm:high-prob}
With probability at least $1-\tfrac{1}{T}$,
Equations~(\ref{eq:algo-estimate-r}-\ref{eq:algo-estimate-c}) hold for all rounds $t$ and  policies $\pi\in\Pi$.
\end{lemma}
} 

\section{Correctness of the algorithm}
\label{sec:correctness}

We need to prove that in each round $t$, some $P\in \mF_t$ satisfies \refeq{eq:alg-balanced}, and Equations~(\ref{eq:algo-estimate-r}-\ref{eq:algo-estimate-c}) hold for all policies $\pi\in\Pi$ with high probability.

\xhdr{Notation.}
Recall that $P_t$ is the distribution over $\Pi$ chosen in round $t$ of the algorithm, and $\noiseProb$ is the noise probability. The ``noisy version" of $P_t$ is defined as
 \begin{align*}
    P_t'(a|x) = (1-\noiseProb)\,P_t(a|x)+\noiseProb/K
        \qquad (\forall x\in X, a\in A).
\end{align*}
Then action $a_t$ in round $t$ is drawn from distribution $P'_t(\cdot|x_t)$.

\begin{lemma}\label{lm:correctness-balance}
In each round $t$, some $P\in \mF_t$ satisfies \refeq{eq:alg-balanced}.
\end{lemma}

\begin{proof}
First we prove that $\mF_t$ is compact; here each distribution over $\Pi$ is interpreted as a $|\Pi|$-dimensional vector, and compactness is w.r.t. the Borel topology on $\R^{|\Pi|}$. This can be proved via standard real analysis arguments; we provide a self-contained proof in Appendix~\ref{app:compactness}.

In what follows we extend the minimax argument from~\cite{policy_elim}. Our proof works for any $\noiseProb\in[0,\tfrac12]$ and any compact and convex set $\mF\subset \FPi$.

Denote
    $\alpha_{\pi} = \max_{P\in \mF} P(\pi)$,
for each $\pi\in \Pi$. Let $\FPi$ be the set of all distributions over $\Pi$.

\eqref{eq:alg-balanced} holds for a given $P\in \mF$ if and only if for every distribution $Z\in \FPi$ we have that
$$f(P,Z) \triangleq \E_{x\sim \DX}\;\E_{\pi\sim Z}
    \left[ \frac{\alpha_{\pi}}{P'(\pi(x)|x)}\right]\leq 2K,
$$
where $P'$ is the noisy version of $P$. It suffices to show that
\begin{align}\label{eq:min-max-leq}
    \min_{P\in\mF} \max_{Z\in \FPi} f(P,Z)\leq 2K.
\end{align}
We use a min-max argument: noting that $f$ is a convex function of $P$ and a concave function of $Z$, by the Sion's minimax theorem \citep{Sion58} we have that
\begin{align}\label{eq:min-max}
    \min_{P\in \mF} \max_{Z \in \FPi} f(P,Z)=\max_{Z\in \FPi} \min_{P\in \mF} f(P,Z).
\end{align}
For each policy $\pi\in \Pi$, let
    $\beta_\pi \in \argmax_{\beta\in \mF} \beta(\pi)$
be a distribution which maximizes the probability of selecting $\pi$. Such distribution exists because $\beta\mapsto \beta(\pi)$ is a continuous function on a compact set $\mF$. Recall that
    $\alpha_\pi = \beta_\pi(\pi)$.

Given any $Z\in\FPi$, define distribution $P_Z\in \FPi$ by
    $P_Z(\pi)=\sum_{\phi\in \Pi}\; Z(\phi) \,\beta_{\phi}(\pi)$.
Note that $P_Z$ is a convex combination of distributions in $\mF$. Since $\mF$ is convex, it follows that $P_Z\in \mF$.  Also, note that
    $ P_Z(a|x) \geq \sum_{\pi\in \Pi:\; \pi(x)=a} Z(\pi)\, \alpha_{\pi}$.
Letting $P'_Z$ be the noisy version of $P_Z$, we have:
\begin{align*}
\min_{P\in \mF} f(P,Z)
    &\leq f(P_Z,Z)
    = \E_{x\sim \DX}\left[\sum_{\pi}
        \frac{Z(\pi)\, \alpha_{\pi}}{P'_Z(\pi(x)|x)} \right] \\
     &= \E_{x\sim \DX}\left[\sum_{a\in A} \;\;
        \sum_{\pi\in \Pi:\;\pi(x)=a} \;
        \frac{Z(\pi)\, \alpha_{\pi}}{P'_Z(a|x)} \right]
    = \E_{x\sim \DX} \left[\sum_{a\in X} \;
        \frac{ \sum_{\pi\in\Pi:\; \pi(x)=a} Z(\pi)\, \alpha_{\pi}}
            {(1-\noiseProb)P_Z(a|x)+\noiseProb/K} \right] \\
    &\leq \E_{x\sim \DX}\left[\sum_{a\in X} \frac{1}{1-\noiseProb}\right]=\frac{K}{1-\noiseProb}
    \leq 2K.
\end{align*}
Thus, by~\eqref{eq:min-max} we obtain~\eqref{eq:min-max-leq}.
\end{proof}

To analyze Equations~(\ref{eq:algo-estimate-r}-\ref{eq:algo-estimate-c}),
we will use Bernstein's inequality for martingales~\citep{Fre75}, via the following formulation from \cite{BestofBoth-colt12}:

\begin{lemma}
\label{lm:martingale-convergence}
Let $\mG_0 \subseteq \mG_1\subseteq \ldots \subseteq \mG_n$ be a filtration, and $X_1,\ldots,X_n$ be real random variables such that $X_t$ is $\mG_t$-measurable, $\E(X_t|\mG_{t-1})=0$ and $|X_t|\leq b$ for some $b>0$. Let $V_n=\sum_{t=1}^n \E(X_t^2|\mG_{t-1})$. Then with probability at least $1-\delta$ it holds that
\begin{align*} 
\textstyle
\sum_{t=1}^n X_t\leq \sqrt{4V_n\log(n\delta^{-1})+5b^2\log^2(n\delta^{-1})}.
\end{align*}
\end{lemma}

\begin{lemma}\label{lm:correctness-estimates}
With probability at least $1-\tfrac{1}{T}$,
Equations~(\ref{eq:algo-estimate-r}-\ref{eq:algo-estimate-c}) hold for all rounds $t$ and  policies $\pi\in\Pi$.
\end{lemma}

\begin{proof}
Let us prove \eqref{eq:algo-estimate-r}. (The proof of \refeq{eq:algo-estimate-c} is similar.) Fix round $t$ and policy $\pi\in\Pi$. We bound the conditional variance of the estimators $\empir{r}_t(\pi)$. Specifically, let $\mG_t$ be the $\sigma$-algebra induced by all events up to (but not including) round $t$. Then
\begin{align*}
\E\left[ \empir{r}_t(\pi)^2 \,|\, \mG_t \right]
= \E_{x\sim\DX,\; a\sim P'_t}\;
        \left[\frac{r_t^2\;\indicator{\pi(x)=a}}{ P'_t(a|x)^2}\right]
\leq \E_{x\sim\DX}
        \left[\frac{1}{ P'_t(\pi(x)|x)}\right]
\leq \frac{2K}{\alpha_{\pi,t}}.
\end{align*}
The last inequality holds by the algorithm's choice of distribution $P_t$.
Since the confidence region $\mI$ in our algorithm is non-increasing over time, it follows that $\alpha_{\pi,t}$ is non-increasing in $t$, too. We conclude that
$ \mathtt{Var}\left[ \empir{r}_s(\pi) \,|\, \mG_s \right]  \leq 2K/\alpha_{\pi,t} $
for each round $s\leq t$. Therefore, noting that
    $\empir{r}_t(\pi)\leq  1/P'(\pi(x_t) | x_t) \leq   K/\noiseProb$,
we obtain \eqref{eq:algo-estimate-r} by applying Lemma~\ref{lm:martingale-convergence} with
    $X_t = \empir{r}_t(\pi) - r(\pi)$.
\end{proof}

\section{Regret analysis: proof of Theorem~\ref{thm:intro-algo}}
\label{sec:analysis}

We provide the key steps of the proof; the details can be found in Section~\ref{app:regret-analysis}.

Let $\mI_t$ and $\F_t$ be, resp., the confidence region $\mI$ and the set $\F$ of potentially LP-perfect distributions computed in round $t$. Let $\Conv(\F_t)$ be the convex hull of $\F_t$.

First we bound the deviations within the confidence region.

\begin{lemma}\label{lm:samedist-differentmu-revcost-convergence}
For any two \eoss $\mu',\mu''\in \mI_t$ and a distribution $P\in \Conv(\F_t)$:
\begin{align}
|c_i(P,\mu')-c_i(P,\mu'')|&\leq \rad_t\left(dK\right)
\qquad \text{for each resource $i$}
    \label{eq:lm:samedist-differentmu-revcost-convergence-c}\\
|r(P,\mu')-r(P,\mu'')|&\leq \rad_t\left(dK\right)
    \label{eq:lm:samedist-differentmu-revcost-convergence-r}
\end{align}
\end{lemma}

\begin{proof}
Let us prove \eqref{eq:lm:samedist-differentmu-revcost-convergence-r}. (\eqref{eq:lm:samedist-differentmu-revcost-convergence-c} is proved similarly.)
By definition of $\mI_t$:
\begin{align*}
|r(P,\mu')-r(P,\mu'')|
&\leq \textstyle  \sum_{\pi\in \Pi}P(\pi)\; |r(\pi,\mu')-r(\pi,\mu'')| \\
&\leq \textstyle \sum_{\pi\in \Pi}P(\pi)\; \rad_t\left(K/\alpha_{\pi,t}\right).
\end{align*}
It remains to prove that the right-hand side is at most $\rad_t(dK)$. By linearity, it suffices to prove this for $P\in \F_t$. So let us assume $P\in \F_t$ from here on. Recall that
$|\support(P)|\leq d$ since $P$ is LP-perfect, and
    $P(\pi)\leq \alpha_{\pi,t}$
for any policy $\pi\in\Pi$. Therefore:
\begin{align*}
\textstyle \sum_{\pi\in \Pi}P(\pi)\; \rad_t\left(K/\alpha_{\pi,t}\right)
&\leq \textstyle \sum_{\pi\in \Pi}\;  \rad_t\left(K P(\pi)\right) \\
&\leq \textstyle \rad_t\left(dK\sum_{\pi\in \Pi} P(\pi)\right)
= \rad_t\left(dK\right). \qedhere
\end{align*}
\end{proof}

Using Lemma~\ref{lm:samedist-differentmu-revcost-convergence} and a long computation (fleshed out in Section~\ref{app:regret-analysis}), we prove the following.

\begin{lemma}\label{lm:lp-samedist-convergence}
For any two \eoss $\mu',\mu''\in \mI_t$ and a distribution $P\in \Conv(\F_t)$:
\begin{align*}
\LP(P,\mu')-\LP(P,\mu'') \leq (\tfrac{1}{B}\,\LP(P,\mu')+2)\cdot T\cdot \rad_t(dK).
\end{align*}
\end{lemma}

Let $\Rew_t$ and $\empirval{c}_{t,i}$ be, respectively, the (realized) total reward and average consumption of resource $i$ up to and including round $t$. Recall that $P'_t$ is the noisy version of distribution $P_t$ chosen by the algorithm in round $t$. Given $P_t$, the expected revenue and resource-$i$ consumption in round $t$ is, respectively, $r(P'_t,\mu)$ and $c_i(P'_t,\mu)$. Denote
    $\avgval{r}_t=\frac{1}{t}\sum_{i=1}^t r(P'_t,\mu)$
and
    $\avgval{c}_{i,t}=\frac{1}{t}\sum_{i=1}^t c_i(P'_t,\mu)$.

\xhdr{Analysis of a clean execution.}
Henceforth, without further notice, we assume a {\bf\em clean execution} where several high-probability conditions are satisfied. Formally, the algorithm's execution is \emph{clean} if in each round $t$
 Equations~(\ref{eq:algo-estimate-r}-\ref{eq:algo-estimate-c}) are satisfied, and moreover
$\min\left(
    | \tfrac{1}{t}\,\Rew_t  - \avgval{r}_t| ,\;
    | \empirval{c}_{t,i} - \avgval{c}_{t,i}|
\right) \leq \rad_t(1).$

In particular, the set $\F_t$ of potentially LP-perfect distributions indeed contains a LP-perfect distribution. By Lemma~\ref{lm:correctness-estimates} and Azuma-Hoeffding Inequality, clean execution happens with probability at least $1-\tfrac{1}{T}$. Thus, it suffices to lower-bound the total reward $\Rew_T$ for a clean execution.

\begin{lemma}\label{lm:sandwich}
For any distribution $P'\in \Conv(\F_t)$ and any \eos $\mu\in \mI_t$,
\begin{align}\label{eq:lm:sandwich}
    \min_{P\in \F_t}\LP(P,\mu)\leq \LP(P',\mu)\leq \max_{P\in \F_t}\LP(P,\mu).
\end{align}
\end{lemma}

\begin{proof}
The proof consists of two parts. The second inequality in~\eqref{eq:lm:sandwich} follows easily because the distribution which maximizes $\LP(P,\mu)$ by definition belongs to $\F_t$, and so
$$\LP(P',\mu)\leq \max_{P\in \Conv(\F_t)}\LP(P,\mu)=\max_{P\in \F_t}\LP(P,\mu).$$

To prove the first inequality in~\eqref{eq:lm:sandwich}, we first argue that $\LP(P,\mu)$ is a quasi-concave function of $P$. Denote
    $\eta_i(P,\mu)=B\cdot r(P,\mu)/c_i(P,\mu)$
for each resource $i$. Then $\eta_i$ is a quasi-concave function of $P$ since each \emph{level set} (the set of distributions $P$ that satisfy $\eta_i(P,\mu)\geq \alpha$ for some $\alpha\in \R$) is a convex set. Therefore $\LP(P,\mu)=\min_i \eta_i(P,\mu)$ is a quasi-concave function of $P$ as a minimum of quasi-concave functions.

Since $P'\in \Conv(\F_t)$, it is a convex combination
    $P'=\sum_{Q\in \F_t}\alpha_Q~Q$
with $\sum_{Q\in \F_t}\alpha_Q=1$.
Therefore:
\begin{align*}
\LP(P',\mu)&=\LP\left(\sum_{Q\in \F_t}\alpha_Q~Q,\; \mu\right) \\
&\geq \min_{Q\in \F_t,\alpha_Q>0}\LP(Q,\mu) &\textrm{By definition of quasi-concave functions}\\
&\geq \min_{Q\in \F_t}\LP(Q,\mu). \qedhere
\end{align*}
\end{proof}

The following lemma captures a crucial argument. Denote
\begin{align*}
\Phi_t =
    \left( 2+\tfrac{1}{B}\,
                \left[ \max_{P\in \mF_{\Pi},\;\mu\in \mI_t} \LP(P,\mu)
                \right]
    \right)\cdot T\cdot \rad_t(dK)
\end{align*}


\begin{lemma}\label{lm:PhiT}
For any \eos $\mu^*,\mu^{**}\in \mI_t$ and distributions $P',P''\in  \Conv(\F_t)$:
\begin{align}\label{eq:lm:PhiT}
    |\LP(P',\mu^*)-\LP(P'',\mu^{**})|\leq 3\Phi_t.
\end{align}
\end{lemma}

\begin{proof}
Assume $P',P''\in \F_t$. In particular, $P',P''$ are LP-perfect for some \eoss $\mu',\mu''\in\mI_t$, resp. Also, some distribution $P^*\in \F_t$ is LP-perfect for $\mu^*$ (by Lemma~\ref{lm:perfect}). Therefore:
\begin{align*}
\LP(P',\mu^*)
    &\geq  \LP(P',\mu')- \Phi_t
        &\textrm{(by Lemma~\ref{lm:lp-samedist-convergence}: $P = P'$)} &\\
    &\geq  \LP(P^*,\mu')- \Phi_t \\
    &\geq  \LP(P^*,\mu^*)- 2\Phi_t
        &\textrm{(by Lemma~\ref{lm:lp-samedist-convergence}: $P =P^*$)} &\\
    &\geq  \LP(P'',\mu^*)- 2\Phi_t.
\end{align*}

We proved \eqref{eq:lm:PhiT} for $P',P''\in \F_t$. Thus:
\begin{align}\label{eq:lm:PhiT-1}
\max_{P\in \F_t}\LP(P,\mu^*) - \min_{P\in \F_t}\LP(P,\mu^*) \leq 2\Phi_t.
\end{align}

Next we  generalize to $P',P''\in \Conv(\F_t)$.
\begin{align*}
\LP(P',\mu^*)
    &\geq \min_{P\in \F_t}\LP(P,\mu^*)
        \qquad\textrm{(by Lemma~\ref{lm:sandwich})} \\
    &\geq \max_{P\in \F_t}\LP(P,\mu^*)-2\Phi_t
        \qquad\textrm{(by  \eqref{eq:lm:PhiT-1})} \\
    &\geq \LP(P'',\mu^*)-2\Phi_t
        \quad\textrm{(by Lemma~\ref{lm:sandwich})}.
\end{align*}
We proved \eqref{eq:lm:PhiT} for $\mu^* = \mu^{**}$. We obtain the general case by plugging in
Lemma~\ref{lm:lp-samedist-convergence}.
\end{proof}

Next, we upper-bound $\Phi_t$ in terms of
\begin{align*}
 \Psi_t = (2+\tfrac{1}{B}\,\OPT_{\LP})\cdot T\cdot \rad_t(dK).
\end{align*}

\begin{corollary}\label{cor:Phi-Psi}
$\Phi_t \leq 2\Psi_t$, assuming that $B\geq 6\cdot T\cdot \rad_t(dK)$.
\end{corollary}
\begin{proof}
Follows from Lemma~\ref{lm:PhiT} via a simple computation, see Section~\ref{app:regret-analysis}.
\end{proof}

\begin{corollary}\label{lm:item:usedlpvalue}
$\LP(P_t,\mu) \geq \LPOPT-12\,\Psi_t$, where $\mu$ is the actual \eos.
\end{corollary}

\begin{proof}
Follows from Lemma~\ref{lm:PhiT} and Corollary~\ref{cor:Phi-Psi}, observing that
    $P_t \in \Conv(\F_t)$
and
    $\LPOPT = \LP(P^*,\mu)$
for some
    $P^* \in \F_t$.
\end{proof}

In the remainder of the proof (which is fleshed out in Section~\ref{app:regret-analysis}) we build on the above lemmas and corollaries to prove the following sequence of claims:
\begin{align}
\Rew_t       &\geq \frac{t}{T}\left(\LPOPT-O(\Psi_t) \right) \nonumber \\
\empirval{c}_{t,i}
            &\leq B/T+O(\rad_t(dK)) \label{eq:regret-analysis-sequence}\\
\Rew_T      &\geq \LPOPT-O(\Psi_{T}). \nonumber
\end{align}

To complete the proof of Theorem~\ref{thm:intro-algo}, we re-write the last equation as
    $\Rew_T \geq f(\LPOPT)$
for an appropriate function $f()$, and observe that $f(\LPOPT)\geq f(\OPT)$ because function $f()$ is increasing.

\section{Regret analysis: remaining details for the proof of Theorem~\ref{thm:intro-algo}}
\label{app:regret-analysis}


\subsection{Proof of Lemma~\ref{lm:lp-samedist-convergence}}
\label{subsec:details-1}

We restate the lemma for convenience.

\xhdr{Lemma~~~} For any two \eoss $\mu',\mu''\in \mI_t$ and a distribution $P\in \Conv(\F_t)$:
\begin{align*}
\LP(P,\mu')-\LP(P,\mu'') \leq (\tfrac{1}{B}\,\LP(P,\mu')+2)\cdot T\cdot \rad_t(dK).
\end{align*}

\begin{proof}
For brevity, we will denote:
\begin{align*}
\LP' = \LP(P,\mu') &\quad\text{and}\quad \LP'' = \LP(P,\mu'') \\
r' = r(P,\mu')    &\quad\text{and}\quad  r'' = r(P,\mu'') \\
c'_i = c_i(P,\mu') &\quad\text{and}\quad c''_i = c_i(P,\mu'').
\end{align*}
By symmetry, it suffices to prove the upper bound for
    $\LP'-\LP''$.
Henceforth, assume $\LP'>\LP''$.

We consider two cases, depending on whether
\begin{align}\label{eq:lm:lplatentconvergence-cases}
T \leq B/c''_i  \quad\text{for all resources $i$}.
\end{align}

{\bf Case 1.} Assume~\eqref{eq:lm:lplatentconvergence-cases} holds.
Then $\LP''=T~r''$. Therefore by Lemma~\ref{lm:samedist-differentmu-revcost-convergence}
\begin{align*}
\LP'-\LP'' \leq T~r'-T~r'' \leq T~\rad_t(dK).
\end{align*}

{\bf Case 2.} Assume~\eqref{eq:lm:lplatentconvergence-cases} fails.
Then
    $\LP''= B~r''/c''_i$
for some resource $i$. We consider two subcases, depending on whether
\begin{align}\label{eq:lm:lplatentconvergence-subcases}
T\leq B/c'_j  \quad\text{for all resources $j$}.
\end{align}

{\bf Subcase 1.} Assume~\eqref{eq:lm:lplatentconvergence-subcases} holds. Then:
\begin{align}
\LP' &= T~r' \label{eq:lm:lplatentconvergence-subcase1-1}\\
\LP'' & \leq  T\cdot \min( r',\; r'') \leq \LP' \label{eq:lm:lplatentconvergence-subcase1-2}
\end{align}
\eqref{eq:lm:lplatentconvergence-subcase1-2} follows from \refeq{eq:lm:lplatentconvergence-subcase1-1} and $\LP'>\LP''$.


For $\delta\in [0, c''_i)$, define
\begin{align*}
r(\delta)   &= r'' + \delta \\
c_i(\delta) &= c''_i -\delta \\
f(\delta)   &= B\, r(\delta) / c_i(\delta).
\end{align*}
Then $f()$ is monotonically and continuously increasing function, with $f(\delta)\to\infty$ as $\delta\to c''_i$. For convenience, define $f(c''_i) = \infty$.

Let $\delta_0 = \min(c''_i,\,\rad_t(dK))$. By Lemma~\ref{lm:samedist-differentmu-revcost-convergence}, we have
    $f(\delta_0) \geq B r'/c'_i$.
Therefore:
\begin{align*}
f(0)= \LP'' <\LP' \leq B r'/c'_i \leq f(\delta_0).
\end{align*}
Thus, by \eqref{eq:lm:lplatentconvergence-subcase1-2}, we can fix $\delta \in [0, \delta_0)$ such that
    $f(\delta)= T\cdot \min( r',\; r'' )$.

\begin{align*}
\LP''
    &=   B\; \frac{r''}{c''_i}
    =   B\; \frac{r(\delta)-\delta}
                     {c_i(\delta)+\delta}\\
    &\geq    B\; \frac{r(\delta)-\delta}
                     {c_i(\delta)}\left(1-\frac{\delta}{c_i(\delta)}\right). \\
f(\delta) - \LP''
    &\leq \frac{B}{c_i(\delta)}\delta+B\frac{r(\delta)}{c_i(\delta)^2}\delta \\
    & = \left(1+\frac{r(\delta)}{c_i(\delta)}\right) \frac{B~\delta}{c_i(\delta)} \\
    & = \left(1+\frac{f(\delta)}{B}\right) \frac{f(\delta)~\delta}{r(\delta)} \\
    & \leq \left(1+\frac{T~r'}{B}\right) \frac{T~r''~\delta}{r(\delta)} \\
    & \leq \left(1+\frac{\LP'}{B}\right) T~\delta \\
    &\leq \left(\LP'/B+1\right)\cdot T\cdot \rad_t(dK). \\
\LP'-f(\delta)
    &=T~r' - T~\min(r,r') \\
    &\leq T\cdot\rad_t(dK) \\
\LP'-\LP''
    &=  \left( \LP'-f(\delta)\right) + \left( f(\delta)-\LP'' \right) \\
    &\leq \left(\LP'/B+2\right)\cdot T\cdot \rad_t(dK).
\end{align*}

{\bf Subcase 2.} Assume~\eqref{eq:lm:lplatentconvergence-subcases} fails. Then
    $\LP'= B~r'/c'_j$
for some resource $j$. Note that
    $c'_i\leq c'_j$ and $c''_j\leq c''_i$
by the choice of $i$ and $j$.

From these inequalities and Lemma~\ref{lm:samedist-differentmu-revcost-convergence} we obtain
    $c''_i\leq c'_j+\rad_t(dK)$.
Therefore,
\begin{align*}
B\frac{r''}{c''_i}
    &\geq  B\; \frac{r'-\rad_t(dK)}
                   {c'_j+\rad_t(dK)}
               \qquad \text{(by Lemma~\ref{lm:samedist-differentmu-revcost-convergence})}  \\
    &\geq  B\; \frac{r'-\rad_t(dK)}
                   {c'_j}\left(1-\frac{\rad_t(dK)}{c'_j}\right). \\
\LP'-\LP''
    &=    B\frac{r'}{c'_j}-B\frac{r''}{c''_i} \\
    &\leq \left(\; \frac{B}{c'_j} + B\frac{r'}{(c'_j)^2} \right) \rad_t(dK) \\
    &\leq \left( T +\LP'\frac{T}{B}  \right) \rad_t(dK) \\
    &\leq \left(\LP'/B+1\right)\cdot T\cdot \rad_t(dK). \qedhere
\end{align*}
\end{proof}

\subsection{Remainder of the proof after Lemma~\ref{lm:PhiT}}
\label{subsec:details-2}

We start with Corollary~\ref{cor:Phi-Psi}, which we restate here for convenience.

\xhdr{Corollary~~~} $\Phi_t \leq 2\Psi_t$, assuming that $B\geq 6\cdot T\cdot \rad_t(dK)$.

\begin{proof}
Let
    $\gamma=\max_{P\in \mF_{\Pi},\mu\in \mI_t}\LP(P,\mu)$.
Note that $\gamma \leq T$.
Then from Lemma~\ref{lm:PhiT} we obtain:
\begin{align}
\gamma-\OPT_{\LP}&\leq 3 (\tfrac{\gamma}{B}+2)\cdot T\cdot \rad_t(dK)
\leq \tfrac{\gamma}{2}+6\cdot T\cdot \rad_t(dK). \label{eq:lm:item:usedlpvalue}
\end{align}

Using \refeq{eq:lm:item:usedlpvalue} and Lemma~\ref{lm:PhiT} we get the desired bound:
\begin{align*}
\Phi_t
    &\leq (\tfrac{\gamma}{B}+2)\cdot T\cdot \rad_t(dK) \\
    &\leq \left(\frac{2\,\OPT_{\LP}+12\cdot T\cdot \rad_t(dK)}{B}+2\right)\cdot T\cdot \rad_t(dK) \\
    &\leq \left(\frac{2\,\OPT_{\LP}}{B}+4\right)\cdot T\cdot \rad_t(dK)
    = 2 \Psi_t. \qedhere
\end{align*}
\end{proof}

In the remainder of this appendix, we prove the claims in \eqref{eq:regret-analysis-sequence} one by one.

\begin{corollary}\label{cor:lowerbound-empirrev}
    $\Rew_t\geq \frac{t}{T}\left(\LPOPT-O(\Psi_t) \right)$
for each round $t\leq \stime$.
\end{corollary}

\begin{proof}
From Lemma~\ref{lm:item:usedlpvalue} we obtain
\begin{align*}
T\,r(P'_t,\mu) &
    \geq (1-\noiseProb)\; \LP(P_t,\mu) \\
   & \geq (1-\noiseProb) \left( \LPOPT-12\Psi_t \right)\\
   & \geq \LPOPT-13\Psi_t.
\end{align*}
Summing up and taking average over rounds, we obtain:
$$ T\,\avgval{r}_t
    \geq \LPOPT-\tfrac{13}{t} \textstyle \sum_{s=1}^t\; \Psi_s
    \geq \LPOPT-O(\Psi_t).$$
By definition of clean execution, we obtain:
\begin{align*}
\Rew_t\geq t(\avgval{r}_t-\rad_t(\avgval{r_t}))
    \geq \tfrac{t}{T}(\LPOPT-O(\Psi_t)).\qedhere
\end{align*}
\end{proof}

\begin{corollary}\label{cor:upperbound-empircost}
    $\empirval{c}_{t,i}\leq B/T+O(\rad_t(dK))$
for each round $t\leq \stime$.
\end{corollary}
\begin{proof}
Let $\mu$ be the (actual) \eos, and recall that $P_t$ is LP-optimal for some \eos $\mu'\in \F_t$. Then, by Lemma~\ref{lm:samedist-differentmu-revcost-convergence}, it follows that
    $c_i(P_t,\mu)\leq c_i(P_t,\mu')+\rad_t(dK)$.
Furthermore since $P_t$ is LP-optimal for $\mu'$ we have $c_i(P_t,\mu')\leq \tfrac{B}{T}$.
Therefore:
\begin{align*}
c_i(P_t,\mu)
    &\leq \tfrac{B}{T}+\rad_t(dK) \\
c_i(P'_t,\mu)
    &\leq (1-\noiseProb)\,c_i(P_t,\mu)+ \noiseProb\\
    &\leq \tfrac{B}{T}+O(\rad_t(dK)).
\end{align*}
Now summing and taking average we obtain
    $\avgval{c}_{t,i} \leq \tfrac{B}{T}+O(\rad_t(dK))$.
Using the definition of clean execution, it follows that
\begin{align*}
\empirval{c}_{t,i}
    &\leq \avgval{c}_{t,i}+\rad_t(\avgval{c}_{t,i})
    \leq \tfrac{B}{T}+O(\rad_t(dK)).\qedhere
\end{align*}
\end{proof}

\begin{lemma}\label{lm:regret-clean}
$\Rew_T\geq \LPOPT-O(\Psi_{T})$.
\end{lemma}

\begin{proof}
Either $\stime=T$ or some resource $i$ gets exhausted, in which case (using Corollary~\ref{cor:upperbound-empircost})
\begin{align}
 \stime
     = \frac{B}{\empirval{c}_{\stime,i}}
    &\geq \frac{B}{\tfrac{B}{T}+\rad_\stime(dK)} \nonumber\\
\Rightarrow \stime \tfrac{B}{T}+\stime \rad_{\stime}(dK) &\geq B \nonumber \\
\Rightarrow \stime \tfrac{B}{T}+T \rad_{T}(dK) &\geq B \nonumber \\
\Rightarrow \stime 	& \geq T\left(1-\tfrac{T}{B}\,\rad_T(dK)\right). \label{eq:stopTime-LB}
\end{align}
Using this lower bound and Corollary~\ref{cor:lowerbound-empirrev}, we obtain the desired bound on the total revenue $\Rew_T$.
\begin{align*}
\Rew_T
    = \Rew_{\stime}
     &\geq \frac{\stime}{T} \left(\, \OPT_{\LP}-O(\Psi_{\stime}) \,\right)\\
    &\geq  \LPOPT(1 -  \tfrac{T}{B} \,\rad_T(dK)) -\frac{O(\stime\, \Psi_{\stime})}{T}\\
    &\geq  \LPOPT - \Psi_T - \frac{O(\stime\, \Psi_{\stime})}{T}.
\end{align*}
In the above, the first inequality holds by Corollary~\ref{cor:lowerbound-empirrev}, the second by \eqref{eq:stopTime-LB}, and the third by definition of $\Psi_T$.

Finally, we note that
    $\stime \,\Psi_{\stime}$ is an increasing function of $\stime$,
and substitute
    $\stime \Psi_{\stime}\leq T\Psi_T$.
\end{proof}

We complete the proof of Theorem~\ref{thm:intro-algo} as follows. Re-writing Lemma~\ref{lm:regret-clean} as
    $\Rew_T \geq f(\LPOPT)$,
for an appropriate function $f()$, note that $\Rew_T\geq f(\OPT)$ because function $f()$ is increasing.

\section{Lower bound: proof of Theorem~\ref{thm:intro-LB}}
\label{sec:LB}

In fact, we prove a stronger theorem that implies Theorem~\ref{thm:intro-LB}.

\begin{theorem}\label{thm:LB-body}
Fix any tuple $(K,T,B)$ such that $K\in [2,T]$ and $B \leq \sqrt{KT}/2$.
Any algorithm for \cBwK incurs regret $\Omega(\OPT(\Pi))$ in the worst case over all problem instances with $K$ actions, time horizon $T$, smallest budget $B$, and policy sets $\Pi$ such that $\OPT(\Pi) \leq B$.
\end{theorem}

We will use the following lemma (which follows from simple probability arguments).

\begin{lemma}\label{lm:lowerbound:supportinglemma}
Consider two collections of $n$ balls $\mathcal{I}_1$ and $\mathcal{I}_2$, each numbered from $1$ to $n$. Let $\mathcal{I}_1$ consists of all red balls, while $\mathcal{I}_2$ consist of $n-1$ red balls and $1$ green ball (with labels chosen uniformly at random). In this setting, let an algorithm is given access to random samples from one of $\mathcal{I}_i$ with replacement. The algorithm is allowed to first look at the ball's number and then decide whether to inspect it's color. Then any algorithm $\mathcal{A}$ which with probability at least $\tfrac{1}{2}$ can distinguish between $\mathcal{I}_1$ and $\mathcal{I}_2$ must inspect color of at least $n/2$ balls in expectation.
\end{lemma}

In the remainder of this section we prove Theorem~\ref{thm:LB-body}.

Let us define a family of problem instances as follows. Let the set of arms be $\{a_1,a_2,\ldots,a_K\}$. There are $T/B$ different contexts labelled $\{x_1,...,x_{T/B}\}$ and there is a uniform distribution over contexts. The policy set $\Pi$ consists of $T~(K-1)/B$ policies $\pi_{i,j}$, where $2\leq i\leq K$ and $1\leq j\leq T/B$. Define them as follows:
$\pi_{i,j}(x_l) =a_i$ for $l=j$,
and $\pi_{i,j}(x_l) =a_1$ for $l\neq j$.

\OMIT{ 
\begin{align*}
\pi_{i,j}(x_l)&=a_i &\textrm{For }l=j\\
&=a_1 &\textrm{For }l\neq j
\end{align*}
} 

There is just one resource constraint $B$ (apart from time). Pulling arm $a_1$ always costs 0 and arm $a_i,i\neq 1$ always costs 1. Now consider the following problem instances:

\fakeItem Let $\mathcal{F}_0$ be the instance in which every arm always gives a reward $0$. Note that  $\OPT(\mathcal{F}_0)=0$.

\fakeItem Let $\mathcal{F}_{i,j}$ be the instance in which arm $a_i$ on context $x_j$ gives reward $1$, otherwise every arm on every context gives reward 0. Note that in this case the optimal distribution over policies is just to follow $\pi_{i,j}$ and gets reward $\approx B$.

Now consider any algorithm $\mathcal{A}$ and let the expected number of times it pulls arm $a_i$ be $p_i$ on input $\mathcal{F}_0$. Let $i',i'\neq 1$ be the arm for which this is minimum. Then by simple linearity of expectation we get that $B\geq (K-1)p_{i'}$. It is also simple to see that for the algorithm to get a regret better than $\Omega(\OPT)$ it should be able to distinguish between $\mathcal{F}_0$ and $\mathcal{F}_{i',.}$ at least with probability $\tfrac{1}{2}$. From lemma~\ref{lm:lowerbound:supportinglemma} this can be done iff $p_{i'}\geq T/(2B)$. Combining the two equations we get $B\geq (K-1)T/(2B)$. Solving for $B$ we get $B\geq \sqrt{KT}/2$.

\section{Discretization for contextual dynamic pricing (proof of Theorem~\ref{thm:discretization})}
\label{app:discretization}

We consider contextual dynamic pricing with $B$ copies of a single product. The action space consists of all prices $p\in [0,1]$. We obtain regret bounds relative to an arbitrary policy set $\Pi$.

\xhdr{Preliminaries.}
Let $S(p|x)$ be the contextual \emph{sales rate}: the probability of a sale for price $p$ and context $x$. Note that $S(p|x)$ is non-increasing in $p$, for any given $x$.

The assumption of Lipschitz demands is stated as follows:
\begin{align}\label{eq:Lipschitz}
|S(p|x)-S(p'|x)| \leq L\cdot |p-p'| \quad \text{for all contexts $x$},
\end{align}
for some constant $L$ called the \emph{Lipschitz constant}. For simplicity, assume $L\geq 1$.

For a (possibly randomized) policy $\pi$, define the contextual sales rate
    $S(\pi|x) = \E_{p\sim \pi(x)}[\, S(p|x) \,]$
and the absolute sales rate 
    $S(\pi) = \E_{x}[\, S(\pi|x) \,]$.
The latter is exactly the expected per-round resource consumption for $\pi$. Let $r(\pi)$ be the expected per-round reward for $\pi$.

As discussed in the Introduction, we define the discretization with step $\eps$ as follows. For each price $p$, let $f_\eps(p)$ be $p$ rounded down to the nearest multiple of $\eps$, i.e. the largest price $p'\leq p$ such that $p'\in \eps\N$. For each policy $\pi$ we define a discretized policy $\pi_\eps = f_\eps(\pi)$. The discretized policy set is then
    $\Pi_\eps = \{\pi_\eps: \pi\in \Pi\}$.
Note that for all policies $\pi$ and all contexts $x$ we have
\begin{align*}
\pi(x) \geq \pi_\eps(x) \geq \pi(x)-\eps.
\end{align*}

\noindent By monotonicity of the sales rate and the Lipschitz assumption, resp., it follows that
\begin{align*}
S(\pi|x) \leq S(\pi_\eps|x) \leq S(\pi|x) + \eps L.
\end{align*}
Consequently, 
    $S(\pi) \leq S(\pi_\eps) \leq S(\pi) + \eps L.$

\xhdr{Discretization error.}
The key technical step is to bound the discretization error of the discretized policy set $\Pi_\eps$ compared to the original policy set $\Pi$, as quantified by the difference in $\LPOPT(\cdot)$.

Our proof will use an intermediate policy class 
    $\Phi_{\delta} = \{ S(\pi) \geq \delta\}$,
where $\delta>0$. First we bound the discretization error relative to $\Phi_{\delta}$.

\newcommand{\myE}[1]{\E_{x,\pi}\left[\, #1 \,\right]}

\begin{lemma}\label{lm:discretization-1}
$\LPOPT(\Phi_\delta) - \LPOPT(\Pi_\eps) \leq 2\cdot\eps(1+L \delta^{-2})\,\cdot B$, 
for each $\eps,\delta>0$.
\end{lemma}

\begin{proof}
Using a trivial reduction to the non-contextual case (when a policy corresponds to an action in the bandits-with-knapsacks problem), one can use a generic discretization result from \citet{BwK-focs13}. According to this result (specialized to contextual dynamic pricing), it suffices to prove that for each policy $\pi\in\Phi_\delta$ the following two properties hold:
\begin{OneLiners}
\item[(P1)] $S(\pi_\eps)\geq S(\pi)$,
\item[(P2)] $r(\pi_\eps)/S(\pi_\eps) \geq r(\pi)/S(\pi)-\eps(1+ L \delta^{-2})$,
    as long as $S(\pi_\eps)>0$.
\end{OneLiners}
In words: the sales rate of the discretized policy $\pi_\eps$ is at least the same, and the reward-to-consumption ratio is not much worse. 

Property (P1) holds trivially because $\pi_\eps\leq \pi$ (deterministically and for every context), and the contextual sales rate $S(p|x)$ is decreasing in $p$ for any fixed context $x$.
\begin{align*}
r(\pi_\eps)   &=      \myE{ f_\eps(\pi(x))  \cdot S(\pi_\eps|x) } \\
            &\geq   \myE{ (\pi(x) - \eps) \cdot S(\pi_\eps|x) } \\
            &\geq   \myE{ \pi(x)\cdot S(\pi|x) } -
                    \eps\, \myE{ S(\pi_\eps|x) } \\
            &= r(\pi) - \eps\, S(\pi_\eps).  \\
r(\pi_\eps)/S(\pi_\eps)
    &\geq r(\pi)/S(\pi_\eps) - \eps.
\end{align*}
Now, by the Lipschitz assumption,
    $S(\pi_\eps)\leq S(\pi)+\eps L $, so to complete the proof
\begin{align*}
\frac{r(\pi_\eps)}{S(\pi_\eps)}
    \geq \frac{r(\pi)}{S(\pi)+\eps L} -\eps  
    \geq \frac{r(\pi)}{S(\pi)} - \frac{\eps L}{ (S(\pi))^2} -\eps
    \geq \frac{r(\pi)}{S(\pi)} - \frac{\eps L}{ \delta^2} -\eps. \qedhere
\end{align*}
\end{proof}

Now we bound the loss in $\LPOPT$ between $\Pi$ and $\Phi_\delta$.
\begin{lemma}\label{lm:discretization-2}
$\LPOPT(\Pi) -\LPOPT(\Phi_\delta) \leq \delta T$, for each $\delta>0$.
\end{lemma}
\begin{proof}
If $\delta\geq B/T$, the statement is trivial because $\LPOPT(\Pi)\leq B$. So w.l.o.g. assume $\delta<B/T$.

By Lemma~\ref{lm:perfect}, there exists an LP-perfect distribution $P$ over policies in $\Pi$. Recall that $P$ is a mixture of (at most) two policies, say $\pi$ and $\pi'$, and $c(P)\leq B/T$. W.l.o.g. assume $S(\pi)\leq S(\pi')$.

If $S(\pi)\geq \delta$ then $\pi,\pi' \in \Phi_\delta$, so 
    $\LPOPT(\Pi) =\LPOPT(\Phi_\delta)$.

The remaining case is $S(\pi)< \delta$. Then 
    $S(\pi')\geq B/T>\delta$, so $\pi'\in \Phi_\delta$.
Therefore:
\begin{align*}
\LPOPT(\Pi) 
    = \LP(P) 
    \leq \LP(\pi)+\LP(\pi')
    \leq \LP(\pi) + \LPOPT(\Phi_\delta).
\end{align*}
It remains to prove that $\LP(\pi)\leq \delta T$. Indeed,
\begin{align*}
r(\pi)
    &= \myE{ \pi(x) \cdot S(\pi|x) }
    \leq \myE{ S(\pi|x) }
    = S(\pi) \leq \delta. \\
\LP(\pi) 
    &= r(\pi)\, \min(T,\,B/S(\pi))
    \leq r(\pi)\, T\leq \delta T.
    \qedhere
\end{align*}
\end{proof}

Putting Lemma~\ref{lm:discretization-1} and Lemma~\ref{lm:discretization-1} together and optimizing $\delta$, we obtain:
\begin{lemma}\label{lm:discretization-3}
For each $\eps>0$, letting $\delta = (2\eps B L/T)^{1/3}$, we have
$$\LPOPT(\Pi) - \LPOPT(\Pi_\eps) \leq 2\delta T  +2\eps B.$$
\end{lemma}

\xhdr{Plugging in the general result.}
Let $\Rew(\Pi')$ be the expected total reward when \CBalance is run with policy set $\Pi'$ which uses only $K$ distinct actions. Recall that we actually prove a somewhat stronger version of Theorem~\ref{thm:intro-algo}: the same regret bound \refeq{eq:thm:intro-algo}, but with respect to $\LPOPT(\Pi')$ rather than $\OPT(\Pi')$. In our setting we have $d=2$ resource constraints (incl. time) and $\LPOPT(\Pi')\leq B$. Therefore:
\begin{align*}
\Rew(\Pi') \geq \LPOPT(\Pi') - O\left( \sqrt{KT \; \log \left( KT\,|\Pi'| \right)}\right).
\end{align*}

Plugging in $\Pi'=\Pi_\eps$ and $K = \tfrac{1}{\eps}$, and using Lemma~\ref{lm:discretization-3}, we obtain
\begin{align}\label{eq:discretization-eps}
\Rew(\Pi_\eps) \geq \LPOPT(\Pi) 
- O\left( \eps B + \delta T + 
   \sqrt{\tfrac{T}{\eps} \; \log \left( \tfrac{T}{\eps}\,|\Pi_\eps| \right)}\right),
\end{align}
for each $\eps>0$ and $\delta = (2\eps B L/T)^{1/3}$.

We obtain Theorem~\ref{thm:discretization} choosing 
    $\eps = (BL)^{-2/5}\, T^{-1/5}\, (\log( T\,|\Pi_\eps|))^{3/5} $ 
and noting
    $|\Pi_\eps| \leq |\Pi|$.

\section{Conclusions and open questions}
\label{sec:conclusions}

We define a very general setting for contextual bandits with resource constraints (denoted \cBwK). We design an algorithm for this problem, and derive a regret bound which achieves the optimal root-$T$ scaling in terms of the time horizon $T$, and the optimal $\sqrt{\log |\Pi|}$ scaling in terms of the policy set $\Pi$. Further, we consider discretization issues, and derive a specific corollary for contextual dynamic pricing with a single product; we obtain a regret bound that applies to an arbitrary policy set $\Pi$. Finally, we derive a partial lower bound which establishes a stark difference from the non-contextual version. These results set the stage for further study of \cBwK, as discussed below.

The main question left open by this work is to combine provable regret bounds and a computationally efficient (CE) implementation. While we focused on the statistical properties, we believe our techniques are unlikely to lead to CE implementations. Achieving near-optimal regret bounds in a CE way has been a major open question for contextual bandits with policy sets (without resource constraints). This question has been resolved in the positive in a simultaneous and independent work \citep{monster-icml14}. Very recently, a follow-up paper \citep{AgrawalDevanurLi-15} has achieved the corresponding advance on \cBwK, by combing the techniques from \citet{monster-icml14} and \citet{AgrawalDevanur-ec14} (which, in turn, builds on \citet{BwK-focs13}).

\OMIT{ 
As a stepping stone, one can target CE algorithms with weaker regret bounds, such as $T^{2/3}$ rather than $\sqrt{T}$ dependence on time, perhaps building on the corresponding results for the non-constrained version \citep{Langford-nips07}.
} 

\vspace{2mm}

Computational issues aside, several open questions concern our regret bounds.

First, it is desirable to achieve the same regret bounds without assuming a known time horizon $T$ (as it is in most bandit problems in the literature). This may be difficult because time is one of the resource constraints in our problem, and our techniques rely on knowing all resource constraints in advance. More generally, one can consider a version of \cBwK in which some of the resource constraints are not fully revealed to an algorithm; instead, the algorithm receives updated estimates of these constrains over time.

Second, while our main regret bound in Theorem~\ref{thm:intro-algo} is optimal in the important regime when $\OPT(\Pi)$ and $B$ are at least a constant fraction of $T$, it is not tight for some other regimes. For a concrete comparison, consider problem instances with a constant number of resources ($d$), a constant number of actions ($K$), and $\OPT(\Pi)\geq \Omega(B)$. Then, ignoring logarithmic factors, we obtain regret $\OPT(\Pi)\, \sqrt{T}/B$, whereas the lower bound in \citet{BwK-focs13} is $\OPT(\Pi)/\sqrt{B}$. So there is a gap when $B\ll T$. 
Likewise, for contextual dynamic pricing with a single product, there is a gap between our algorithmic result (Theorem~\ref{thm:discretization}) and the $B^{2/3}$ lower bound for the non-contextual case from \citet{DynPricing-ec12}. In both cases, both upper and lower bounds can potentially be improved.

Third, for special cases when actions correspond to prices one would like to extend the discretization approach beyond contextual dynamic pricing with a single product. However, this is problematic even without contexts: essentially, nothing is known whenever one has multiple resource constraints, and even with a single resource constraint (besides time) the solutions are very non-trivial; see \citet{BwK-focs13} for more discussion.

\OMIT{ 
in some applications of \cBwK, such as dynamic pricing and dynamic procurement, the action space can be a continuous interval of prices. Theorem~\ref{thm:intro-algo} usefully applies whenever the policy set $\Pi$ is chosen so that the number of distinct actions used by policies in $\Pi$ is finite and small compared to $T$. (Because one can w.l.o.g. remove all other actions.) However, one also needs to handle problem instances in which the policies in $\Pi$ use prohibitively large or infinite number of actions. 
} 

Fourth, if there are no contexts or resource constraints then one can achieve $O(\log T)$ regret with an instance dependent constant; it is not clear whether one can meaningfully extend this result to contextual bandits with resource constraints.

\vspace{2mm}

The model of \cBwK can be extended in several directions, two of which we outline below. The most immediate extension is to an unknown distribution of context arrivals. This extension has been addressed, among other results, in the follow-up paper \citep{AgrawalDevanurLi-15}. The most important extension, in our opinion, would be from a stationary environment to one controlled by an adversary (perhaps restricted in some natural way). We are not aware of any prior work in this direction, even for the non-contextual version.

\bibliography{bib-abbrv,bib-slivkins,bib-bandits,bib-AGT,bib-random}

\appendix

\section{\cBwK: applications and special cases}
\label{app:apps}

In this section we discuss the application domains of resource-constrained (contextual) bandits in more detail. We focus on the three main application domains: dynamic pricing, dynamic procurement, and dynamic ad allocation. A more extensive discussion of these and other application domains (in the non-contextual version) can be found in \cite{BwK-focs13,BwK-full}.

\xhdr{Dynamic pricing with limited supply.} The algorithm is a monopolistic seller with a limited inventory. In the basic version, there is a limited supply of identical items. In each round, a new customer arrives, the algorithm picks a price, and offers one item for sale at this price. The customer then either buys the item at this price, or rejects the offer and leaves. The ``context" represents the available information about the current customer, such as demographics, location, etc. The probability of selling at a given price for a given context (a.k.a. the \emph{demand distribution}) is fixed over time, but not known to the algorithm. The algorithm optimizes the revenue; it does not derive any utility from  the left-over items.

We represent this problem as an instance of \cBwK as follows. ``Actions" are the possible prices, and the ``resource constraint" is the number of items.  In each round, the outcome vector is a pair (reward, items sold); if the offered price is $p$, the outcome vector is $(p,1)$ if there is a sale, and $(0,0)$ otherwise.

Many generalizations of dynamic pricing have been studied in the literature. In particular, \cBwK subsumes a number of extensions. First, an algorithm can sell multiple items to the same customer, possibly with volume discounts or surcharges. Second, an algorithm can have multiple products for sale, with limited inventory of each. Third, it may be advantageous to offer bundles consisting of different products, possibly with non-additive pricing (mirroring the non-additive valuations of the customers).

\xhdr{Dynamic procurement on a budget.} The algorithm is a monopolistic buyer with a limited budget. The basic version is as follows. In each round, a new customer arrives, the algorithm picks a price, and offers to buy one item at this price. Then the customer either accepts the offer and sells the item at this price, or rejects the offer and leaves. The ``context" is the available information on the current customer. The probability of buying at a given price for a given context (a.k.a. the ``supply distribution") is fixed over time, but not known to the algorithm. The algorithm maximizes the number of items bought; it has no utility for the left-over money.

An alternative interpretation is that the algorithm is a contractor which hires workers to perform tasks, e.g. in a crowdsourcing market. In each round, a new worker arrives, the algorithm picks a price, and offers the worker to perform one task for this price; the worker then either accepts and performs the task at this price, or rejects and leaves. The relevant ``context" for a worker in a crowdsourcing market may include, for example, age, location, language, and task preferences.

Here, ``actions" correspond to the possible prices, and the ``resource constraint" is the buyer's budget.  In each round, the outcome vector is a pair (items bought, money spent); if the offered price is $p$, then the outcome vector is $(1,p)$ if the offer is accepted, and $(0,0)$ otherwise.
 
Dynamic procurement is a rich problem space, both for buying items and for hiring workers (see \citep{Crowdsourcing-PositionPaper13} for a discussion of the application to crowdsourcing markets). In particular, \cBwK subsumes a number of extensions of this basic setting. First, the algorithm may offer several tasks to the same worker, possibly at a discount. Second, there may be multiple types of tasks, each having a different value for the contractor; moreover, there may be additional budget constraints on each task type, or on various \emph{subsets} of task types. Third, a given worker can be offered a bundle of tasks, consisting of tasks of multiple types, possibly with non-additive pricing. Fourth, there is a way to model the presence of competition (other contractors).
 
\xhdr{Dynamic ad allocation with budgets.} The algorithm is an advertising platform. In the basic version, there is a fixed collection of ads to choose from. In each round, a user arrives, and the algorithm chooses one ad to display to this user. The user either clicks on this ad, or leaves without clicking. The algorithm receives a payment if and only if the ad is clicked; the payment for a given ad is fixed over time and known to the algorithm. The ``context" is the available information about the user and the page on which the ad is displayed. The click probability for a given ad and a given context is constant over time, but not known to the algorithm. 

Each ad belongs to some advertiser (who is the one paying the algorithm when this ad is clicked). Each advertiser may own multiple ads, and has a budget constraint: a maximal amount of money that can be spent on all his ads. Moreover, an advertiser may specify additional budget constraints on various subsets of the ads. The algorithm maximizes its revenue; it derives no utility from the left-over budgets.

Here, ``actions" correspond to ads, and each budget corresponds to a separate resource. In a round when the chosen ad $a$  is clicked, the reward is the corresponding payment $v$, and the resource consumption is $v$ for each budget that involves $a$, and $0$ for all other budgets. If the ad is not clicked, the reward and the consumption of each resource is $0$.

\cBwK also subsumes more advanced versions in which multiple non-zero outcomes are possible in each round. For example, the ad platform may record what happens \emph{after} the click, e.g. the time spent on the page linked from the ad and whether this interaction has resulted in a sale.

\section{Compactness of $\mF_t$}
\label{app:compactness}

Recall that $\PotPerf_t$ is the set of distributions over $\Pi$ that is computed by our algorithm in each round $t$, and
$\mF_t = \Conv(\PotPerf_t)$ is the convex hull of $\PotPerf_t$. In this appendix we prove that $\mF_t$ is compact for each $t$. Here each distribution over $\Pi$ is interpreted as a $|\Pi|$-dimensional vector, and compactness is with respect to the Borel topology on $\R^{|\Pi|}$.

\begin{lemma}\label{lm:app-compactness}
$\mF_t$ is compact for each $t$, relative to the Borel topology on $\R^{|\Pi|}$.
\end{lemma}

Since a convex hull of compact set is compact, it suffices to prove that $\PotPerf_t$ is compact, i.e. that it is closed and bounded. Each distribution is contained in a unit cube, hence bounded. Thus, it suffices to prove that $\PotPerf_t$ is a closed subset of $\R^{|\Pi|}$.

We use the following general lemma, which can be proved via standard real analysis arguments.

\newcommand{\tX}{{\mathtt X}}
\newcommand{\tY}{{\mathtt Y}}
\newcommand{\mX}{\mathcal{X}}
\newcommand{\mY}{\mathcal{Y}}

\begin{lemma}\label{lm:app-compactness-general}
Consider the following setup:
\begin{OneLiners}
\item $\mX,\mY$ are compact subsets of finite-dimensional real spaces $\R^{d_\tX}$ and $\R^{d_\tY}$, respectively.
\item Functions $f,g_1 \LDOTS g_d: \mX\times \mY\to [0,1]$ are continuous w.r.t. product topology on $\mX\times \mY$.
\item $H(y) = \{ x \in \mX: f(x,y) = \sup_{x'\in \mX} f(x',y) \text{ and } g(x,y)\leq 0 \}$, for each $y\in \mY$.
\end{OneLiners}
Then $H(\mY) = \cup_{y\in \mY} H(y)$ is a closed subset of $\R^{d_\tX}$.
\end{lemma}

We apply this lemma to prove that $\PotPerf_t$ is a closed subset of $\R^{|\Pi|}$. Specifically, we take $\mX$ to be the set of all distributions over policies with support at most $d$, and $\mY$ be the confidence region in round $t$ of the algorithm. It is easy to see that both sets are closed and bounded by definition, therefore compact. Further, for each distribution $P\in \mX$, each \eos $\mu\in \mY$, and each resource $i$ we define $f(P,\mu)$ to be the corresponding LP-value, and $g_i(P,\mu) = c_i(P,\mu)-B/T$. Then $P\in H(\mu)$ if and only if $P$ is an LP-perfect distribution with respect to $\mu$, and $H(\mY)=\PotPerf_t$.

This completes the proof of Lemma~\ref{lm:app-compactness}.

\subsection{Proof of Lemma~\ref{lm:app-compactness-general}}
\label{susbec:compactness}

Suppose $x^*\in \R^{d_\tX}$ is an accumulation point of $H(\mY)$, i.e. there is a sequence $x_1, x_2,\ldots  \in H(\mY)$ such that $x_j \to x^*$. Note that $x^*\in \mX$ since $\mX$ is closed. We need to prove that $x^* \in H(\mY)$.

For each $j$, there is $y_j\in \mY$ such that $x_j \in H(y_j)$. Recall that $\mY$ is closed and bounded. Since $\mY$ is bounded, sequence $\{y_j\}_{j\in\N}$ contains a convergent subsequence. Since $\mY$ is closed, $\mY$ contains the limit of this subsequence. From here on, let us focus on this convergent subsequence.

Thus, we have proved that there exists a sequence of pairs $\{(x_j,y_j)\}_{j\in\N}$ such that
\begin{OneLiners}
\item $x_j \in H(y_j)$ and $y_j\in \mY$ for all $j\in\N$,
\item $x_j \to x^* \in \mX$ and $y_j \to y^* \in \mY$.
\end{OneLiners}
We will prove that $x^* \in H(y^*)$. For that, we need to prove two things:
(i) $g_i(x^*,y^*)\leq 0$ for each $i$, and
(ii) $f(x^*,y^*) = \sup_{x\in \mX} f(x,y^*)$.

First, $g_i(x^*,y^*)\leq 0$ for each $i$ because $g_i(x^*,y^*) = \lim_j g_i(x_j, y_j) \leq 0$ by continuity of $g_i$.

Second, we claim that $f(x^*,y^*) = \sup_{x\in \mX} f(x,y^*)$. For the sake of contradiction, suppose
    $\eps \triangleq \sup_{x\in \mX} f(x,y^*) - f(x^*,y^*) >0$.
By continuity of $f$, the following holds:
\begin{OneLiners}
\item $f(x^{**},y^*) = \sup_{x\in \mX} f(x,y^*)$ for some $x^{**}\in \mX$.
\item there exists an open neighborhood $S$ of $(x^*,y^*)$ on which $|f(x,y) - f(x^*,y^*)| < \eps/4$.
\item there is an open neighborhood $S'$ of $(x^{**},y^*)$ on which $|f(x,y) - f(x^{**},y^*)| < \eps/4$.
\end{OneLiners}
In particular, there are open balls $B_\tX,B'_\tX\subset \mX$ and $B_\tY \subset \mY$ such that
    $(x^*,y^*) \in B_\tX\times B_\tY \subset S$
and
    $(x^{**},y^*) \in B'_\tX \times B_\tY \subset S'$.
For a sufficiently large $j$ it holds that $x_j\in B_\tX$ and $y_j\in B_\tY$. It follows that
    $(x_j, y_j) \in S$ and $(x^{**},y_j) \in S'$.
Therefore:
\begin{align*}
f(x^*,y^*)
    &> f(x_j,y_j) - \eps/4      &\text{(since $(x_j,y_j)\in S$)}  \\
    &\geq f(x^{**},y_j) - \eps/4   &\text{(using the optimality of $x_j$)} \\
    &> f(x^{**},y^*)- \eps/2   &\text{(since $(x^{**},y_j)\in S'$)}.
\end{align*}
We obtain a contradiction which completes the proof.

\end{document}